\documentclass{article}

\usepackage[utf8]{inputenc} 
\usepackage[T1]{fontenc}    
\usepackage{hyperref}       
\usepackage{url}            
\usepackage{booktabs}       
\usepackage{amsfonts}       
\usepackage{nicefrac}       

\usepackage{microtype}
\usepackage{graphicx}
\usepackage{booktabs} 

\usepackage{bbm}
\usepackage{url}
\usepackage{color}
\usepackage{amsthm}
\usepackage{booktabs}
\usepackage{amsmath, amssymb}
\usepackage{mathtools}
\usepackage{amsfonts}
\usepackage{algorithm}
\usepackage[noend]{algorithmic}
\usepackage{amssymb, amsmath,charter, latexsym}
\usepackage{enumerate}

\usepackage{geometry}
\usepackage{algorithm}
\usepackage[noend]{algorithmic}

\newtheorem{definition}{Definition}
\newtheorem{proposition}{Proposition}
\newtheorem{lemma}{Lemma}
\newtheorem{corollary}{Corollary}
\newtheorem{theorem}{Theorem}

\newtheorem{assumption}{Assumption}

\theoremstyle{remark} 
\newtheorem{remark}{Remark}

\usepackage{hyperref}
\usepackage[switch]{lineno}

\usepackage[english]{babel}

\usepackage{natbib}
\usepackage{amsmath,amsthm,amstext,amsfonts,amssymb,mathrsfs, bbm}
\usepackage{color,algorithm,algorithmic} 
\usepackage{caption}
\usepackage{subcaption}

\usepackage{amsmath}
\usepackage{amsthm}
\usepackage{amssymb}
\usepackage{bbm}
\usepackage{mwe}

\DeclareMathOperator*{\argmin}{arg\,min}
\newcommand{\norm}[1]{\left\lVert#1\right\rVert}

\usepackage{xcolor}

\newcommand{\proj}{\mathop{\Pi}}
\newcommand{\R}{\mathbb{R}}

\newcommand{\X}{\mathcal{X}}

\title{Safe Online Convex Optimization with Unknown Linear Safety Constraints}
\author{Sapana Chaudhary and Dileep Kalathil \\ Department of Electrical and Computer Engineering \\Texas A\&M University 
\thanks{{\small Email: \texttt{\{sapanac,dileep.kalathil\}@tamu.edu}}}}
\date{}

\begin{document}

\maketitle

\begin{abstract}%
We study the problem of safe online convex optimization, where the action at each time step must satisfy a set of linear safety constraints. The goal is to select a sequence of actions to minimize the regret without violating the safety constraints at any time step (with high probability). The parameters that specify the linear safety constraints are unknown to the algorithm. The algorithm has access to only the noisy observations of constraints for the chosen actions. We propose an algorithm, called the {Safe Online Projected Gradient Descent} (SO-PGD) algorithm, to address this problem. We show that, under the assumption of the availability of a safe baseline action, the SO-PGD algorithm achieves a regret $O(T^{2/3})$. While there are many algorithms for online convex optimization (OCO) problems with safety constraints available in the literature, they  allow constraint violations during learning/optimization, and the focus has been on characterizing the cumulative constraint violations. To the best of our knowledge, ours is the first work that provides an algorithm with provable guarantees on the regret, without violating the linear safety constraints (with high probability) at any time step.
\end{abstract}


\section{Introduction}

Online learning/optimization is a sequential decision making paradigm, where the decision maker adaptively selects a sequence of actions based on the past observations \citep{cesa2006prediction}. Online convex optimization (OCO) is an important class of online optimization problems, where the cost function faced by the decision maker at each time step is an arbitrarily-varying convex function \citep{hazan2016introduction, shalev2011online}. In the OCO problem, a sequence of arbitrarily-varying convex cost functions $\{f_{t}, t = 1, \ldots, T\}$ are revealed, one per time step, to the decision maker. The decision maker selects an action $x_{t}$ from a convex set $\X$, \textit{before} the cost function $f_{t}$ is revealed. The typical performance objective is to minimize the regret, which  characterizes the difference between cumulative cost incurred by the decision maker and that of an oracle algorithm that employs the best fixed action in hindsidght  at all time steps.  There are a number of OCO algorithms that achieve different sublinear regret guarantees with different computational complexity \citep{hazan2016introduction}. 

In many real-world applications, however, the actions selected by the decision maker must  satisfy some necessary safety constraints over the set $\X$. For example, in power systems, the control actions that decide the demand management should not violate the line flow and the voltage regulation constraints \citep{dobbe2020learning}. In communication networks, the transmission rate is limited by constraints on the maximum allowable radiated power due to interference and human safety considerations \citep{luong2019applications}. In robotics applications, the control actions should maintain the closed-loop stability of the system \citep{aastrom2010feedback}. Typically, such constraints are represented as a \textit{safe set} $\mathcal{X}^{s}$ and the control action $x_{t}$ must lie inside $\mathcal{X}^{s}$ \textit{for all} $t$ for the safe operation of the system. 

Often, the safe set $\mathcal{X}^{s}$ is  determined by the parameters of the system that are typically \textit{unknown} to the decision maker a-priori. For example, in power systems, the constraints on the  control actions depend on the line parameters, which are typically unknown. In robotics, designing a closed-loop stable controller requires the  dynamic model of the robot, which may be unknown. Thus, the decision maker has to learn the unknown parameters to characterize the unknown safe set. While an exploration algorithm can be used to estimate these parameters, such algorithms often take random actions for efficient estimation that may violate the safety constraints. Moreover, taking actions with respect to an estimated safe set may still violate the safety constraints due to the unavoidable estimation errors.

In this paper, we address the problem of \textit{safe online convex optimization} with an \textit{unknown safe set}, where the decision maker has access to only noisy observations of the safety constraints (that define the safe set) for the chosen actions. Our goal is to design an algorithm that minimizes the regret \textit{while satisfying the safety constraints at all time steps}.

While there are many works in the OCO literature that address the problem  with safety constraints (see the related works section below), they typically  allow constraint violations during learning/optimization. The main goal of such algorithms is then   to obtain a (sublinear)  bound on the cumulative constraint violations, in addition to the standard regret.  In sharp contrast to such works, we focus on designing  an algorithm that satisfies the safety constraints  \textit{ at all time steps} while providing a provable guarantee on the regret. 

In this paper, we restrict ourselves to the setting where the unknown safe set is a closed polytope characterized by a set of linear inequalities with unknown parameters. We believe that addressing the linear  constraint  setting is a natural first step towards developing a fundamental understanding of safe OCO algorithms for general non-linear setting. To the best of our knowledge, this is the first work that addresses the safe OCO problem with a provable guarantee on satisfying the safety constraints at all time steps, even in a setting with linear constraints.

\subsection{Related Work}

\textbf{OCO:} The  OCO problem was  first formally addressed in  \citep{zinkevich2003online}, though some prior works  \citep{cesa1996worst, gordon1999regret} had considered similar settings. In \citep{zinkevich2003online}, the author proposed  an online gradient  descent algorithm and  showed that it  achieves  $O(\sqrt{T})$ regret. A number of OCO algorithms under different assumptions have been developed since, see the monographs \citep{shalev2011online, hazan2016introduction}.


\paragraph{OCO with Long Term Constraints:} Most of the standard OCO algorithms assume full knowledge of the constraint set $\mathcal{X}^{s}$.  However, in many real-world applications, the constraint set $\mathcal{X}^{s}$ is often specified in terms of the functional inequalities, i.e., $\mathcal{X}^{s} = \{x \in \mathcal{X} : g_{i,t}(x) \leq 0,  i \in  \{1, \ldots, m\},  t \in  \{1, \ldots, T\}  \}$, where $g_{i,t}$s are convex functions.  The OCO with long term constraints problem considers a relaxed version of such  constraints, where the goal is to bound the  constraint violations, $\max_{i}~ \sum^{T}_{t=1} g_{i, t}(x_{t})$, instead of satisfying the constraints at each time step. 


The OCO with long term constraint problem was first introduced in  \citep{mahdavi2012trading}, which assumed that  the constraint functions are  the same for all $t$, i.e., $g_{i,t} = g_{i}, \forall t$. For deterministic constraints, the algoirthm proposed in  \citep{mahdavi2012trading} achieves $O(T^{1/2})$ regret and $O(T^{3/4})$  constraint violation. Recently, \citep{yu2020low} showed that it is possible to  achieve $O(T^{1/2})$ regret and $O(1)$  constraint violation. In \citep{yu2017online}, the authors addressed the stochastic constraints setting, where the constraint functions are of the form $g_{i,t}(x) = g_{i}(x, \omega_{t})$, where $w_{t}$s are i.i.d. random variables. They proposed an algorithm that simultaneously achieves $O(\sqrt{T})$ regret and (expected) constraint violation. A recent work \citep{wei2020online} has improved this result  by  removing some assumptions  while maintaining the regret guarantees. 



In \citep{neely2017online}, the authors addressed the setting   where the constraint functions $g_{i,t}$s are arbitrarily-varying (adversarial), and  proposed an algorithm with  $O(T^{1/2})$ regret and constraint violation. This problem was also addressed in \citep{sun2017safety, chen2017online, cao2018online}. A distributed version of this problem has  been studied recently in \citep{yi2020distributed}.

We emphasize that all the above mentioned   works allow constraint violations during learning/optimization. Significantly different from these,  we propose an algorithm that does not violate the unknown linear constraints that define the safe set at any time step during learning/optimization.

\paragraph{Safe Learning/Optimization:} The works closest to our setting are   \citep{amani2019linear} and \citep{khezeli2020safe}, where the authors addressed the linear bandits problem with unknown linear safety constraints that have to be satisfied  at all time steps during learning. For ensuring safe exploration in the initial phase of  learning, they introduce an  assumption about the availability of  a known safe baseline action. They showed that $O(T^{1/2})$ regret is achievable without safety constraints violations during learning if a  lower bound on the distance between the optimal action and the boundary of the safe set is known. If such a lower bound is not available,  then $O(T^{2/3})$ regret is achievable.   Instead of the static linear cost function  considered in these works, we consider the more challenging arbitrarily-varying convex cost functions. Moreover, we also consider a set of linear constraints as opposed to a single linear constraint studied in these works.  

Convex optimization with unknown linear safety constraints   addressed in \citep{usmanova2019safe} and  \citep{fereydounian2020safe} is another class of works that is close to ours. Similar to  \citep{amani2019linear, khezeli2020safe} these works also make use of the assumption of a safe baseline action. They consider  a static convex cost function and focus on characterizing the sample complexity, which is quite different from our setting (arbitrarily-varying cost functions) and objective (regret minimization).

\subsection{Main Contributions}
We formulate the  safe online convex optimization problem  where the action must satisfy a set of \textit{unknown} linear  safety constraints \textit{at all time  steps}. The decision maker has only access to a noisy measurement of the constraints with respect to the chosen action  at each time step. We propose a new algorithm, called the {Safe Online Projected Gradient Descent (SO-PGD) algorithm}, and show that this  algorithm achieves $O(T^{2/3})$ regret  while satisfying the safety constraints at all time steps, with a high probability. To the best of our knowledge, this is the first such result in the OCO literature, even in a setting with liner constraints.

Similar to  \citep{amani2019linear, khezeli2020safe, usmanova2019safe, fereydounian2020safe}, our algorithm  also makes use of the assumption of a safe baseline action for initial exploration and for estimating the unknown parameters. However, a naive estimate of the safe set may lead to constraint violations because of the inherent estimation error. The key idea we use is the construction of a conservative safe set that is provably a subset of the unknown safe set. Our algorithm performs online gradient descent with respect to this conservative safe set, which provably ensures that  safety constraints are satisfied at each time step. We then characterize the error because of using this conservative safe set. We show that a clever balancing of the exploration and online optimization can achieve $O(T^{2/3})$ regret without constraint violations at any time steps.

\subsection{Notations}

For any positive semidefinite matrix $A$, we denote $\|x\|_{A} = \sqrt{x^{\top} A x}$. For any square matrix $A$, we denote its minimum and maximum eigenvalues by $\lambda_{\min}(A)$ and  $\lambda_{\max}(A)$, respectively. For any two integer $M_{1}, M_{2}$ with $M_{1} < M_{2}$, we denote  $[M_{1}, M_{2}] = \{M_{1}, M_{1}+1,  \ldots, M_{2}\}$. For any random vector $\zeta$, $\textnormal{Cov}(\zeta) = \mathbb{E}[\zeta \zeta^{\top}]$. For any convex set $\mathcal{X} \subset \mathbb{R}^{n}$ and any $x \in  \mathbb{R}^{n}$, $\Pi_{\mathcal{X}}(x)$ denotes the projection of $x$ to $\mathcal{X}$ with respect to the Euclidean norm. 


\section{Safe Online Convex Optimization: Problem Formulation}
The general framework of online convex optimization  \citep{hazan2016introduction} is as follows:  at each  time step $t$, the algorithm selects an action $x_{t} \in \mathcal{X} \subset \mathbb{R}^{d}$ and incurs a cost $f_{t}(x_{t})$, where $f_{t}: \mathcal{X} \rightarrow \mathbb{R}$ is a  convex function. The cost function $f_{t}$ is not known at the time of making the decision $x_{t}$, and the sequence of cost functions $\{f_{t}, t \in [1, T]\}$ is assumed to be arbitrary. In addition to the incurred cost  $f_{t}(x_{t})$, it is generally assumed that the value of the gradient of $f_{t}$ evaluated at $x_{t}, $ $\nabla f_{t}(x_{t})$, is also available to the algorithm. The goal of a standard  online convex optimization algorithm is to select a sequence of actions $\{x_{t}, t \in [1, T]\}$ in order to minimize the regret defined as  $\sum_{t=1}^{T} f_t(x_t) -\min_{x \in \mathcal{X}}\sum_{t=1}^{T} f_t(x) $. Most of the existing works  assume that the  set $\mathcal{X}$ is known to the algorithm a priori.

In this work, we consider the safe online convex optimization problem with an unknown safe set characterized by  a set of {unknown linear safety constraints}. More precisely, at each time step $t$, the algorithm has to take an action $x_{t}$ 
from the {safe set} $\mathcal{X}^{s}$, defined as 
\begin{equation} 
\label{eq:safeset}
    \mathcal{X}^{s} = \{x\in \mathcal{X}: A x \leq b\},
\end{equation}
where the matrix $A \in \R^{m \times d}$ and the vector $b \in \mathbb{R}^{m}$.  Denoting $A = [a_1, a_2, \ldots, a_m]^\top, b = [b_1, b_2, \dots, b_m]^{\top}$, where $a_{i} \in \mathbb{R}^{d}$  and  $b_{i} \in \mathbb{R}^{1}$,   the safe set $\mathcal{X}^{s}$ is defined in terms of $m$ linear constraints, and the $i$th linear constraint is of the form  $a^\top_{i} x \leq  b_i$.  We assume that $\mathcal{X}^{s}$ is closed polytope. The matrix $A$ is unknown to the  algorithm a priori. So, the safe set $\mathcal{X}^{s}$ is also unknown. For simplifying the analysis, we assume that $b$ is known to the algorithm.

It is impossible to learn the safety constraints if the algorithm receives no information that can be used to estimate the unknown safe set $\mathcal{X}^{s}$, or equivalently, the unknown   parameter $A$. Here, we make a natural assumption that the algorithm receives a noisy  observation $y_{t} \in \mathbb{R}^{m}$ at each time step $t$, where $y_{t} = A x_{t} + w_{t},$ and $w_{t}$ is a zero mean sub-Gaussian noise.

The goal of the safe online convex optimization algorithm is to select a sequence of actions $\{x_{t}, t \in [1, T]\}$ in order to minimize the regret $R(T)$, defined as 
\begin{align}
\label{eq:regret-defn-1}
R(T) = \sum_{t=1}^{T} f_t(x_t) - \min_{x \in \mathcal{X}^{s}}\sum_{t=1}^{T} f_t(x),
\end{align}
while simultaneously satisfying the safety constraints by ensuring that
\begin{align}
\label{eq:safety-constraints-2}
\mathbb{P}(x_{t} \in \mathcal{X}^{s}, \text{for all}~t \in [1, T]) \geq (1-\delta),
\end{align} 
for a given $\delta \in (0, 1)$.

\subsection{Model Assumptions}

In order to analyze the safe OCO problem stated above,  we make the following  assumptions.

\begin{assumption}[Cost Functions]
\label{as:cost-function}
The  cost functions $\{f_t, t \in [1, T]\}$ are convex and have  a  bounded gradient, i.e.,  $\max_{t \in [1, T]} \max_{x \in \mathcal{X}} \norm{\nabla f_{t}(x)} \leq G$. 
\end{assumption}
The above assumption is  standard in the OCO literature. Also, this assumption implies that $f_t$s are $G$-Lipschitz.

\begin{assumption}[Boundedness]
\label{as:boundedness}
(i) The set  $\mathcal{X}$ is convex and compact. Moreover, $\|x\|_{2} \leq L, \forall x \in \mathcal{X}$. \\
(ii)  $\max_{i \in [1, m]} \norm{a_{i}}_{2} \leq L_{A}$.
\end{assumption}
These are also standard assumptions in the linear bandits and OCO literature. Also, as is standard in the literature, we assume that $G, L, L_{A}$  are known to the algorithm.

\begin{assumption}[Sub-Gaussian Noise]
\label{as:sub-gaussian}
The noise sequence $\{w_{t}, t \in [1, T] \}$ is $R$-sub-Gaussian with respect to a filtration $\{\mathcal{F}_{t}, t \in [1, T]\}$, i.e.,\\
(i) $\mathbb{E}[w_{t}|\mathcal{F}_{t-1}] = 0, \forall t, t \in [1, T],$  \\
(ii) $\mathbb{E}[e^{\lambda w_t}\,|\,\mathcal{F}_{t-1}] \leq \exp(\lambda^2 R^2/2),  \forall \lambda \in \mathbb{R},  \forall  t \in [1, T].$
\end{assumption}

Since the safe set $\mathcal{X}^{s}$ is unknown, clearly it is not possible to satisfy safety constraints right from the first time step without making any additional assumptions. We overcome this obvious limitation by assuming that the algorithm has access to a safe baseline action $x^{s}$ such that $x^{s} \in \mathcal{X}^{s}$. We formalize this assumption as follows.
\begin{assumption}[Safe Baseline Action]
\label{as:safe-baseline-action}
There exists a {safe baseline action} $x^{s} \in \mathcal{X}^{s}$ such that $A x^{s}  = b^{s} < b$. The algorithm knows $x^{s}$ and $b^{s}$ and hence the safety gap $\Delta^{s} = \min_{i} (b_{i} - b^{s}_{i})$.
\end{assumption}
This assumption is similar to that of  the safe baseline action assumption used in the context of safe linear bandits and safe convex optimization \citep{amani2019linear, khezeli2020safe, usmanova2019safe, fereydounian2020safe}. The key intuition is that, any algorithm used in a real-world decision making problem has to perform at least as well as a baseline action, which is often conservatively designed to satisfy the safety constraints. Typically, this baseline action is already employed to solve the real-world decision making problem and there will be large amount of data generated according to this baseline action, which can be used to estimate the value $b^{s}$. We emphasize that while the  baseline action is safe by definition, it may be far way from the optimal action that minimizes the regret.

\section{Safe Online Projected Gradient Descent (SO-PGD) Algorithm}

We propose an algorithm, which we call the \textit{safe online projected gradient descent} (SO-PGD) algorithm, to solve the online convex optimization problem with unknown linear safety constraints. The SO-PGD Algorithm is formally given in Algorithm \ref{alg:sopgd-full}. It has three main parts: (i) safe exploration, (ii) conservative safe set estimation, and (iii) online gradient descent. 

\begin{algorithm}
	\caption{SO-PGD Algorithm}	
	\label{alg:sopgd-full}
	\begin{algorithmic}[1]
	   \STATE \textbf{Input:}  $\gamma, \eta, T_{0}, \delta, x^{s}, T$
	   \STATE \textbf{Safe exploration:}
	   \FOR {$t = 1, \ldots, T_0$}
	   \STATE Select action $x_{t} = (1 - \gamma) x^{s} + \gamma \zeta_{t} $ 
	   \ENDFOR
	   \STATE \textbf{Conservative safe set estimation:}
	   \STATE Estimate $\hat{A}$  according to \eqref{eq:A-hat-estimate}
	   \STATE Compute  conservative safe set $\hat{\mathcal{X}}^{s}$ according to \eqref{eq:conservative-safe-set}
	   \STATE \textbf{Online gradient descent:}
	   \FOR {$t = T_0+1, \ldots, T$}
	   \STATE $x_{t+1} = \Pi_{\hat{\mathcal{X}}^{s}}(x_{t} - \eta \nabla f_{t}(x_{t}))$
	   \ENDFOR 
	\end{algorithmic}
\end{algorithm}

\subsection{Safe Exploration}

The goal of the safe exploration part of the SO-PGD algorithm is to estimate the safe set without violating the safety constraints. This is achieved by pursuing a pure exploration strategy for the first $T_{0}$ time steps by carefully selected exploration actions.  Since the safety constraints have to be satisfied at all time steps, we make use of the knowledge of the safe baseline action $x^{s}$ to collect the observations that are necessary for estimating the safe set. However, since  $y^{s} = A x^{s}$ may not be a function of all the elements of $A$,  taking the safe baseline action $x^{s}$ alone will not give a good estimate of the unknown parameter $A$. To overcome this issue, we design exploration actions as random perturbation around $x^{s}$ in such a way that  they do not  violate the safety constraints. More formally, for any time step $t \in [1, T_{0}]$, the   safe exploration  action $x_{t}$ is selected as
\begin{align}
    \label{eq:safe-exploration-action-1}
    x_{t} = (1 - \gamma) x^{s} + \gamma \zeta_{t},
\end{align}
for some $\gamma \in [0, 1)$,
where $\zeta_{t}$s  are i.i.d. zero mean random vectors such that $\norm{ \zeta_{t}} \leq \min\{1, L\}$  and $\text{Cov}(\zeta_{t}) = \sigma^{2}_{\zeta} I$ for all $t$. By controlling the value of $\gamma$, we can ensure that the exploration action $x_{t}$ satisfies the safety constraints for all $t \in [1, T_{0}]$, as shown below. 
\begin{lemma}
\label{lem:safe-exploration-lemma-1}
Let Assumption \ref{as:boundedness} and \ref{as:safe-baseline-action} hold. 
Let $\gamma = \frac{\Delta^{s}}{L_{A}}$. Then, the safe exploration action $x_{t}$ given in \eqref{eq:safe-exploration-action-1} satisfies the safety constraints $A x_{t} \leq b$ for all $t \in [1, T_{0}]$ almost surely. 
\end{lemma}

\subsection{Estimation of Conservative Safe Set}

At the end of the safe exploration phase, using the past exploration actions $x_{t}$ and the past observations $y_{t} = Ax_t + w_t$, $t \in [1, T_{0}]$, the algorithm computes the $\ell_{2}$-regularized least squares estimate $\hat{A}$ of the matrix $A$. More formally, let $X_{T_{0}} = [x_{1}, \ldots, x_{T_{0}}]^{\top} \in \mathbb{R}^{T_{0} \times d}$ and $Y_{T_{0}} = [y_{1}, \ldots, y_{T_{0}}]^{\top} \in \mathbb{R}^{T_{0}}$. Then, the $\ell_{2}$-regularized least squares estimate is given by 
\begin{align}
\label{eq:A-hat-estimate}
    \hat{A} = (\lambda I + X^{\top}_{T_{0}} X_{T_{0}})^{-1} X^{\top}_{T_{0}} Y_{T_{0}}.
\end{align}
We denote $\hat{A} = [\hat{a}_1,\hat{a}_2,\dots,\hat{a}_m]^{\top}$, where $\hat{a}_{i}$ is the estimate of $a_{i}$.

The SO-PGD algorithm next constructs the ellipsoidal confidence set $\mathcal{C}_i(\delta)$ around $\hat{a}_{i}, i \in [1, m],$ that contains the unknown parameter $a_i$ with a probability greater than $(1 - \delta/m)$. More formally, we define
\begin{align}
    \label{eq:confidence-region-1}
    \mathcal{C}_{i}(\delta) = \{a \in \mathbb{R}^{d} : \norm{ a -  \hat{a}_{i}}_{V_{T_{0}}} \leq \beta_{T_{0}}(\delta) \},
\end{align}
where $V_{T_{0}}$ is the Gram matrix of the least squares estimation, given by $V_{T_{0}} = \lambda I + X^{\top}_{T_{0}} X_{T_{0}} = \lambda I + \sum^{T_{0}}_{t=1} x_{t} x^{\top}_{t}$, and 
\begin{align}
\label{eq:beta-defn-1}
    \beta_{T_{0}}(\delta) = R\sqrt{d \log\left(\frac{1+T_0L^2/\lambda}{\delta/m}\right)} + \sqrt{\lambda}L_A .
\end{align}

The radius $\beta_{T_{0}}(\delta)$ of the confidence of set $\mathcal{C}_{i}(\delta)$ is selected in order to to ensure that the true parameter $a_{i}$ is inside it with high probability. We note that this is a standard approach used in the linear bandits literature \cite[Theorem 2]{abbasi2011improved}. We formally state this result below.

\begin{lemma}
\label{lem:confidence-ball}
Let Assumption \ref{as:boundedness} and \ref{as:sub-gaussian} hold. Then, $\mathbb{P}(a_{i} \in \mathcal{C}_{i}(\delta), \forall i \in [1, m]) \geq 1 - \delta$. 
\end{lemma}

Now, using the confidence sets $ \mathcal{C}_{i}(\delta), i \in [1, m]$, the algorithm constructs a conservative safe set  $\hat{\mathcal{X}}^{s}$ as
\begin{align}
\label{eq:conservative-safe-set}
    \hat{\mathcal{X}}^{s} = \{x \in \mathbb{R}^{d}: \tilde{a}^{\top}_{i} x \leq b, \forall \tilde{a}_{i} \in \mathcal{C}_{i}(\delta), \forall i \in [1, m]  \}. 
\end{align}

Note that the elements of $\hat{\X}^{s}$ satisfy the safety constraint  with respect to \textit{all} elements of the confidence set $\mathcal{C}_{i}( \delta), \forall i \in [1, m]$. This condition naturally leads to a conservative inner approximation of the true safe set $\mathcal{X}^{s}$. We formally state this  observation below. 

\begin{lemma}
\label{lem:conservative-subset}
Let Assumption \ref{as:boundedness} and \ref{as:sub-gaussian} hold. Then,
$\hat\X^s \subseteq  \X^s$ with probability at least $1-\delta$.  
\end{lemma}

Using the conservative safe set $\hat{\mathcal{X}}^{s}$ given in \eqref{eq:conservative-safe-set}  as the feasible set in a projected gradient descent algorithm may appear intractable because the constraint $\tilde{a}^{\top}_{i} x \leq b$ has to be satisfied for all $\tilde{a}_{i} \in \mathcal{C}_{i}(\delta)$. However, using the structure of $\mathcal{C}_{i}(\delta)$, it can be shown that \cite[Chapter 19 ]{lattimore2020bandit}  $\hat{\X}^s$ has a more tractable representation as follows
\begin{align}
\label{eq:conservative-safe-set-tractable}
    \hat{\X}^{s} &= \{x \in \mathbb{R}^{d}: \hat{a}_i^\top x + \beta_{T_0}(\delta) \left\|x \right\|_{V_{T_0}^{-1}} \leq b_i, \forall i \in [1, m]\}.
\end{align}
We will use the above representation, both for implementing our algorithm and analyzing its regret guarantees.

\subsection{Online Projected Gradient Descent}

After the initial safe exploration for the first $T_{0}$ time steps and computing the conservative safe set  $\hat{\X}^{s} $, the SO-PGD algorithm performs  online projected gradient descent for $t \in [T_{0}+1, T]$ by treating $\hat{\X}^{s} $ as the feasible set. Formally, the SO-PGD algorithm takes the sequence of  actions $\{x_{t}, t  \in [T_{0}+1, T]\}$ given by
\begin{align}
    x_{t+1} = \Pi_{\hat{\mathcal{X}}^{s}}(x_{t} - \eta \nabla f_{t}(x_{t})).
\end{align}

Since $\hat{\mathcal{X}}^{s}$ is a subset of the true safe set $\mathcal{X}^{s}$, the sequence of actions taken by the SO-PGD algorithm is  safe by definition.

\subsection{Main Result}

We now give the main result of our paper.

\begin{theorem}
\label{thm:main-theorem}
Let Assumptions \ref{as:cost-function}-\ref{as:safe-baseline-action} hold. Consider the SO-PGD algorithm with $\gamma$ as specified in Lemma \ref{lem:safe-exploration-lemma-1},  $\eta = 2L/G \sqrt{T}$ and $T_{0} = T^{2/3}$. 
 Let $\{x_{t}, t \in [1, T]\}$ be the sequence of actions generated by the SO-PGD algorithm. Then, for any $T \geq \left( \frac{\sqrt{8} \beta_{T}(\delta) L}{\gamma \sigma \Delta^{s}} \right)^{3}$,  with a probability greater than $(1 - \delta)$, we have
\begin{align*}
x_{t} \in \mathcal{X}^{s},~ \forall t \in [1, T],~~\textnormal{and}
\end{align*}
\begin{align}
    R(T) &\leq 2 L G T^{2/3} + 2 L G \sqrt{T}  +\frac{L G \sqrt{8 d} \beta_{T}(\delta)}{C(A,b) \sqrt{\gamma^{2} \sigma^{2}_{\zeta}}}  T^{2/3},
\end{align}
where $C(A, b)$ is a positive constant that depends only on the matrix $A$ and vector $b$.  
\end{theorem}

\begin{remark}
Theorem \ref{thm:main-theorem} guarantees that the SO-PGD algorithm achieves  $O(T^{2/3})$  regret , excluding the $O(\log T)$ factor resulting from $\beta_{T}(\delta)$. This is similar to the  $O(T^{2/3})$  regret guarantee obtained in \citep{amani2019linear} for the  safe linear bandits problem. We emphasize that the $O(T^{1/2})$ regret guarantees for safe linear bandits obtained in \citep{amani2019linear, khezeli2020safe} require additional assumption. In particular, they use the knowledge of a lower bound on the distance between the optimal action and the boundary of the safe set. This is not a meaningful assumption in the OCO setting with arbitrarily-varying cost functions. Designing an algorithm that can achieve a better regret without any additional assumptions is an exciting open question. 
\end{remark}

\section{Regret Analysis} \label{sec:regret_analysis}

We analyze the regret of the SO-PGD algorithm by decomposing it into three terms as follows:
\begin{align}
    R(T) &= \underbrace{\sum^{T_{0}}_{t=1} f_t(x_t) - f_t(x^\star)}_{\text{Term I}} + \underbrace{\sum^{T}_{t=T_{0}+1} f_t(x_t) - f_t(\hat{x}^{\star})}_{\text{Term II}}  +\underbrace{\sum^{T}_{t=T_{0}+1} f_t(\hat{x}^{\star}) - f_t(x^\star)}_{\text{Term III}}, \label{eq:SO-PGD_regret_decompos}
\end{align}
where $x^{\star} = \argmin_{x \in \mathcal{X}^{s}} \sum_{t=1}^T f_t(x)$ is the optimal action in hindsight with respect to the true safe set $\mathcal{X}^{s}$ and $\hat{x}^{\star} = \Pi_{\hat{\mathcal{X}}^{s}}(x^{\star})$ is the projection of $x^{\star}$ to the conservative safe set $\hat{\mathcal{X}}^{s}$.   The first term accounts for the regret due to the safe exploration phase. The second term characterizes the regret of a standard online projected gradient descent algorithm with respect to the conservative safe set $\hat{\mathcal{X}}^{s}$. The third term accounts for the error due to using the  conservative safe set $\hat{\mathcal{X}}^{s}$ in the online projected gradient descent  instead of the  true safe set $\mathcal{X}^{s}$. We separately analyze the regret of each term and show that the  regret is ${O}(T^{2/3})$.

\subsection{Regret of Term I}
We bound Term I  as follows:
\begin{align}
\label{eq:regret-tem-1}
\sum^{T_{0}}_{t=1} f_t(x_t) - f_t(x^\star) \leq \sum^{T_{0}}_{t=1} G \norm{x_{t} - x^{\star}} \leq 2 L G  T_{0},
\end{align}
where the first inequality is from Assumption \ref{as:cost-function} and the second inequality is by Assumption \ref{as:boundedness}. Now, by selecting $T_{0}$ as $T^{2/3}$ as specified in  Theorem \ref{thm:main-theorem}, the regret due to Term I will  be ${O}(T^{2/3})$.  

\subsection{Regret of Term II}

We bound this term using the online projected gradient descent analysis \citep{hazan2016introduction} with respect to the estimated safe $\hat{\mathcal{X}}^{s}$. The regret due to Term II is given by the following proposition. 

\begin{proposition} 
\label{prop:opgd-regret}
Let Assumptions \ref{as:cost-function} and \ref{as:boundedness}  hold. Let the learning rate be $\eta = 2L/G \sqrt{T}$. Then,
\begin{align}
     \sum_{t=T_0+1}^T f_t(x_t) - f_t(\hat{x}^{\star}) ~\leq 2 L G T^{1/2}.
\end{align}
\end{proposition}

So, the regret due to Term II will be ${O}(T^{1/2})$, which is order-wise smaller than the regret due to Term I. 

\subsection{Regret of Term III}

The key step here is to bound  $\norm{\hat{x}^\star - x^{\star}}$ as a (decreasing) function of $T_{0}$. We can then use the fact that $T_{0} = T^{2/3}$ to get the net  regret due to this term. We start by making use of the `shrunk polytope' idea used in \citep{fereydounian2020safe}. Consider the  `shrunk  polytope' $\mathcal{X}^{s}_{\text{in}}$ defined as
\begin{align}
\label{eq:inner-polytope}
    \X^{s}_{\textnormal{in}} &= \{x \in \mathbb{R}^{d}:  a_i^\top x + \tau_{\textnormal{in}} \leq b_i, \forall i \in [1, m]\},
\end{align}
where $\tau_{\text{in}}$ is a positive scalar. It is straight forward to note that  if $\tau_{\text{in}}$ is smaller than some constant, $\X^{s}_{\textnormal{in}}$ will be non-empty and will be a `shrunk version' of $\mathcal{X}^{s}$. More precisely,  $ \X^{s}_{\textnormal{in}} $ will be a closed polytope with its faces parallel to the faces of $\mathcal{X}^{s}$, and will be a strict subset of $\mathcal{X}^{s}$. The key objective for defining this  `shrunk polytope'  is to characterize  the distance $\|\Pi_{\X^{s}_{\text{in}}}(x^{\star}) - x^{\star}\|$ in terms of $\tau_{\textnormal{in}}$, which  will then be used to bound the distance $\|\Pi_{\hat{\mathcal{X}}^{s}}(x^{\star}) - x^{\star}\| = \norm{\hat{x}^\star - x^{\star}}$. Note that, our algorithm, however, will not be able to (and does not need to) compute $ \X^{s}_{\textnormal{in}}$ because $a_{i}$s are unknown. We are  using $ \X^{s}_{\textnormal{in}}$ only for the purpose of regret analysis. 

We will  use the following result from \citep{fereydounian2020safe} to  characterize  the distance $\|\Pi_{\X^{s}_{\text{in}}}(x^{\star}) - x^{\star}\|$.

\begin{lemma}[Lemma 1 in \citep{fereydounian2020safe}]
\label{lem:fereydounian2020safe-lemma}
Consider a positive constant $\tau_{\textnormal{in}}$ such that $\X^{s}_{\textnormal{in}}$ is non-empty. Then, for any $x \in \mathcal{X}^{s}$,
\begin{align}
    \|\Pi_{\X^{s}_{\textnormal{in}}}(x) - x\| \leq \frac{\sqrt{d} \tau_{\textnormal{in}}}{C(A, b)},
\end{align}
where $C(A, b)$ is a positive constant that depends only on the matrix $A$ and the vector $b$. 
\end{lemma}

We will now show that  the shrunk polytope $\X^{s}_{\textnormal{in}}$ is  non-empty and  is a subset of the conservative safe set $\hat{\mathcal{X}}^{s}$ for $\tau_{\textnormal{in}} = {2\beta_{T_0}(\delta)L}/{\sqrt{\lambda_{\min}({V_{T_0})}}}$. This also will immediately imply that $\norm{\Pi_{\hat{\mathcal{X}}^{s}}(x^{\star}) - x^{\star}} \leq  \norm{\Pi_{\X^{s}_{\textnormal{in}}}(x^{\star}) - x^{\star}}$. We state this result formally below.

\begin{lemma}
\label{lem:shrunk-polytope-subset}
Let Assumptions \ref{as:boundedness} and \ref{as:sub-gaussian} hold. Let $\tau_{\textnormal{in}} = {2\beta_{T_0}(\delta)L}/{\sqrt{\lambda_{\min}({V_{T_0})}}}$ and $ T_{0} \geq \frac{8 \beta^{2}_{T}(\delta) L^{2}}{\gamma^{2}\sigma^{2}_{\zeta} (\Delta^{s})^{2}}$. Then,  $\X^{s}_{\textnormal{in}}$ is non-empty and  $\X^{s}_{\textnormal{in}} \subseteq \hat{\mathcal{X}}^{s}$, with a probability greater than $(1 - 2 \delta)$. Moreover, $\norm{\Pi_{\hat{\mathcal{X}}^{s}}(x^{\star}) - x^{\star}} \leq  \norm{\Pi_{\X^{s}_{\textnormal{in}}}(x^{\star}) - x^{\star}}$ with a probability greater than $(1 - 2 \delta)$.
\end{lemma}

Using the above lemma, we can now characterize the regret due to Term III as stated in the proposition below. 
\begin{proposition}
\label{prop:term-3-regret}
Let Assumptions \ref{as:cost-function} - \ref{as:sub-gaussian} hold. Then, for $  T_{0} \geq \frac{8 \beta^{2}_{T}(\delta) L^{2}}{\gamma^{2}\sigma^{2}_{\zeta} (\Delta^{s})^{2}}$, with a probability greater than $(1 - 2 \delta)$,
\begin{align}
     \sum^{T}_{t=T_{0}+1} f_t(\hat{x}^\star) - f_t(x^\star) \leq   \frac{L G \sqrt{8 d} \beta_{T}(\delta)}{C(A,b) \sqrt{\gamma^{2} \sigma^{2}_{\zeta}}}   \frac{T}{\sqrt{  T_{0}}}.
\end{align}
\end{proposition}

Note that, when we use $T_{0} = T^{2/3}$ in the above result, we get the regret due to Term III as $O(T^{2/3})$.

The proof of our main theorem can now be obtained by adding the regret due to Terms  I, II, and III.

\section{Simulation Results}
\label{sec:simulation-main}
In this section, we analyze the performance of our SO-PGD algorithm through experiments in  two different settings.

\paragraph{Experiment Setting:} 
We consider  a closed polytope of the form $\X^s = \{x\in \R^{2}: -x_{\max} \leq x_i \leq x_{\max}, i=1,2\}$ as the safe set. It is straight forward to see that the corresponding parameters are $A = [1,0; -1,0; 0,1; 0,-1]$ and $b = x_{\max}\times[1;1;1;1]$. We consider two sequences of functions,  $f_{1,t}$ and $f_{2,t}$, given by 
\begin{align*}
  f_{1,t}(x) = c_t \cdot (\sum_{i=1}^d x_i) + 1,~~ f_{2,t}(x) = \frac{1}{2} \|x - c_t \bar{x}\|^2_2,
\end{align*}
where $c_t$ is a real number  drawn i.i.d. from the set $[c_{\text{lower}}, c_{\text{upper}} ]$. We select $c_{\text{lower}}~\text{and}~c_{\text{upper}}$   appropriately from $\{0.5,1,1.5,2\}$ for different experiment settings. For $f_{2,t}$, $\bar{x}$ is randomly sampled from a standard Gaussian distribution, then  normalized and scaled by 2.5. 
The constraint noise sequence $w_{t}$s are  i.i.d. Gaussian with zero mean and covariance matrix $10^{-3} I$. Note that $f_{1,t}$s are linear function $f_{2,t}$s are a $1$-strongly convex function. 


For a fixed $T$, we first generate the sequence $c_{t}, t \in [1, T]$. Then, we find the optimal action in hindsight, $x^{\star}_{i} = \arg\min_{x \in \X^s} \sum_{t=1}^{T} f_{i,t}(x),~ i=1,2$, using a standard non-linear optimization function like \verb+fmincon+ from \verb+MATLAB+.  

We choose $\lambda = 0.5$ and $\delta = 10^{-3}$. Exploration noise is $\zeta_{t}$s are generated according to a standard Gaussian distribution and then normalized. The safe baseline action is selected  randomly from the set $\mathcal{X}^{s}$. We run the experiments with $T = 10^{6}$ and $T_{0} =  T^{2/3} = 10^{4}$.  We emphasize that for these values, the condition $T_0 \geq \frac{8 \beta^2 L^2}{\gamma^2 \sigma_{\zeta}^2 (\Delta^2)^2}$ specified in Theorem \ref{thm:main-theorem} is satisfied.


\paragraph{Safe exploration:} Figure \ref{fig:exploration} shows the safe baseline action and the actions taken during the safe exploration phase. As guaranteed by Lemma \ref{lem:safe-exploration-lemma-1}, all actions are strictly inside the safe set. 

\begin{figure}
\centering
\subcaptionbox{}{\includegraphics[width=0.3\textwidth,trim=120 225 150 230, clip]{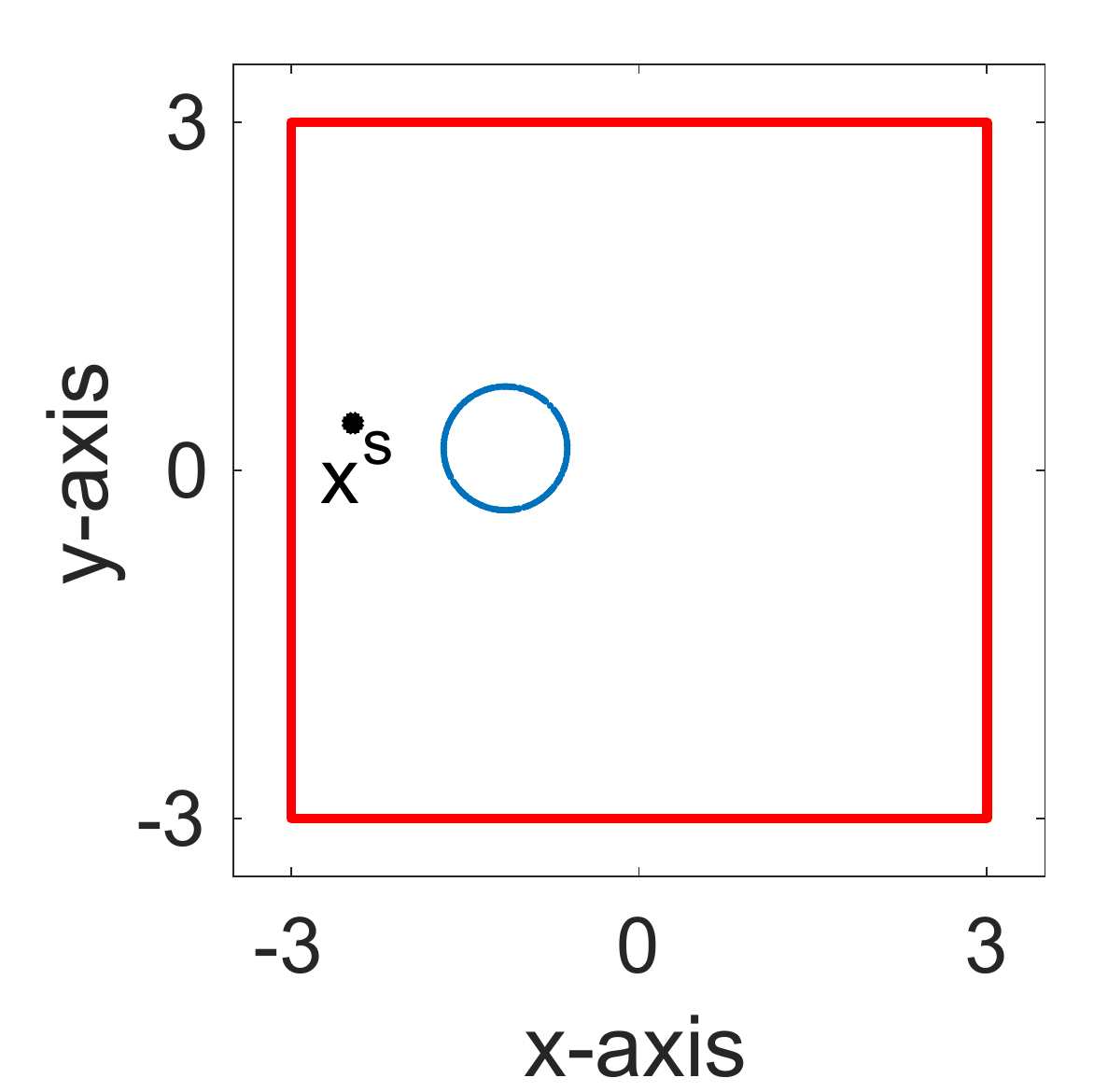}}%
\subcaptionbox{}{\includegraphics[width=0.3\textwidth,trim=120 225 150 230, clip]{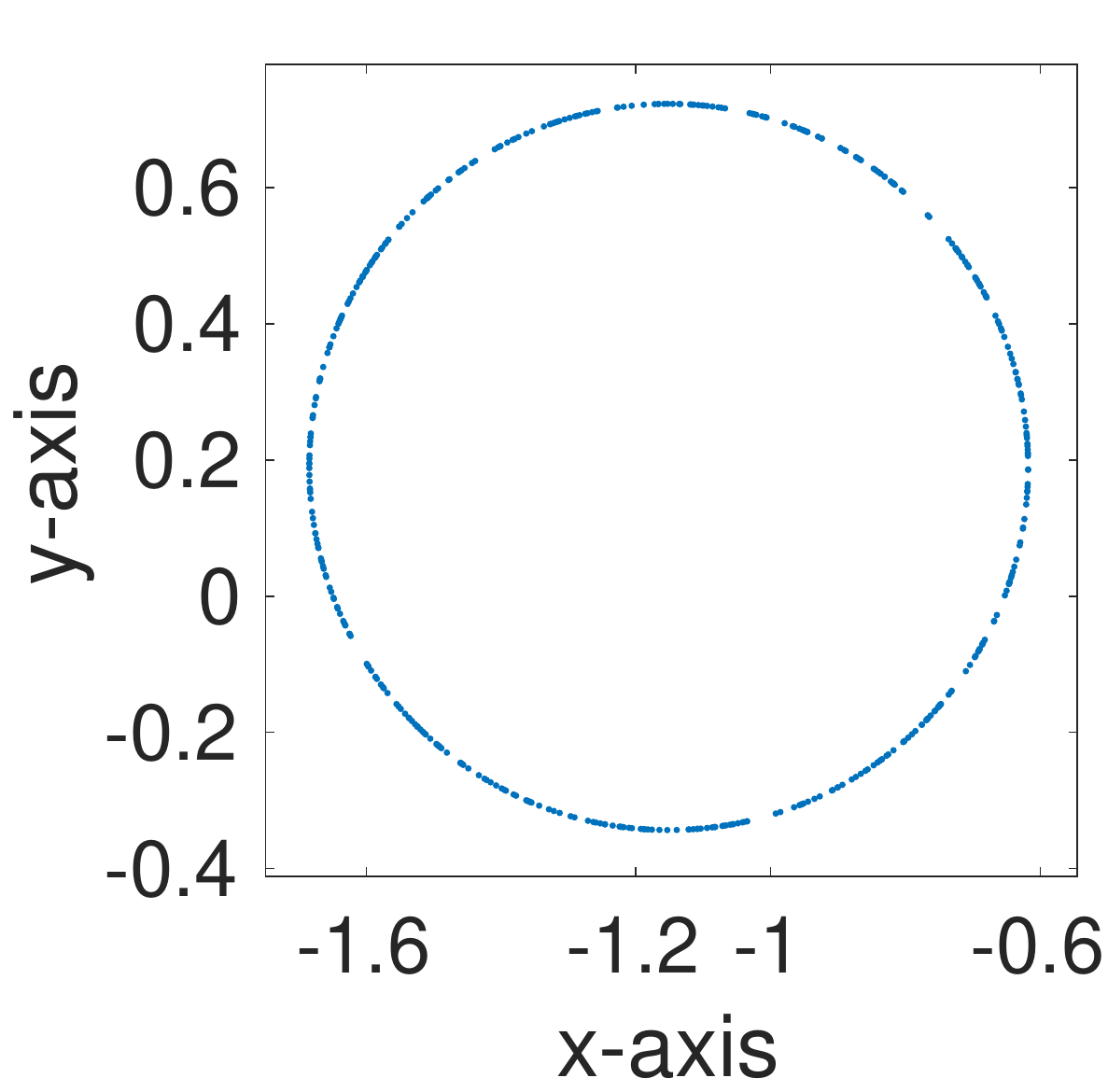}}%
\caption{(a) Shows the  safe set $\mathcal{X}^{s}$ (marked in red), safe baseline action $x^{s}$, and  actions taken during the safe exploration phase (marked in blue).  All the actions lie  inside $\X^s$. (b) Shows zoom-in view of the actions shown in (a).}
\label{fig:exploration}
\end{figure}


\paragraph{Conservative safe set estimation:} 
\begin{figure}
\centering
\subcaptionbox{}{\includegraphics[width=0.35\textwidth,trim=50 180 70 210, clip]{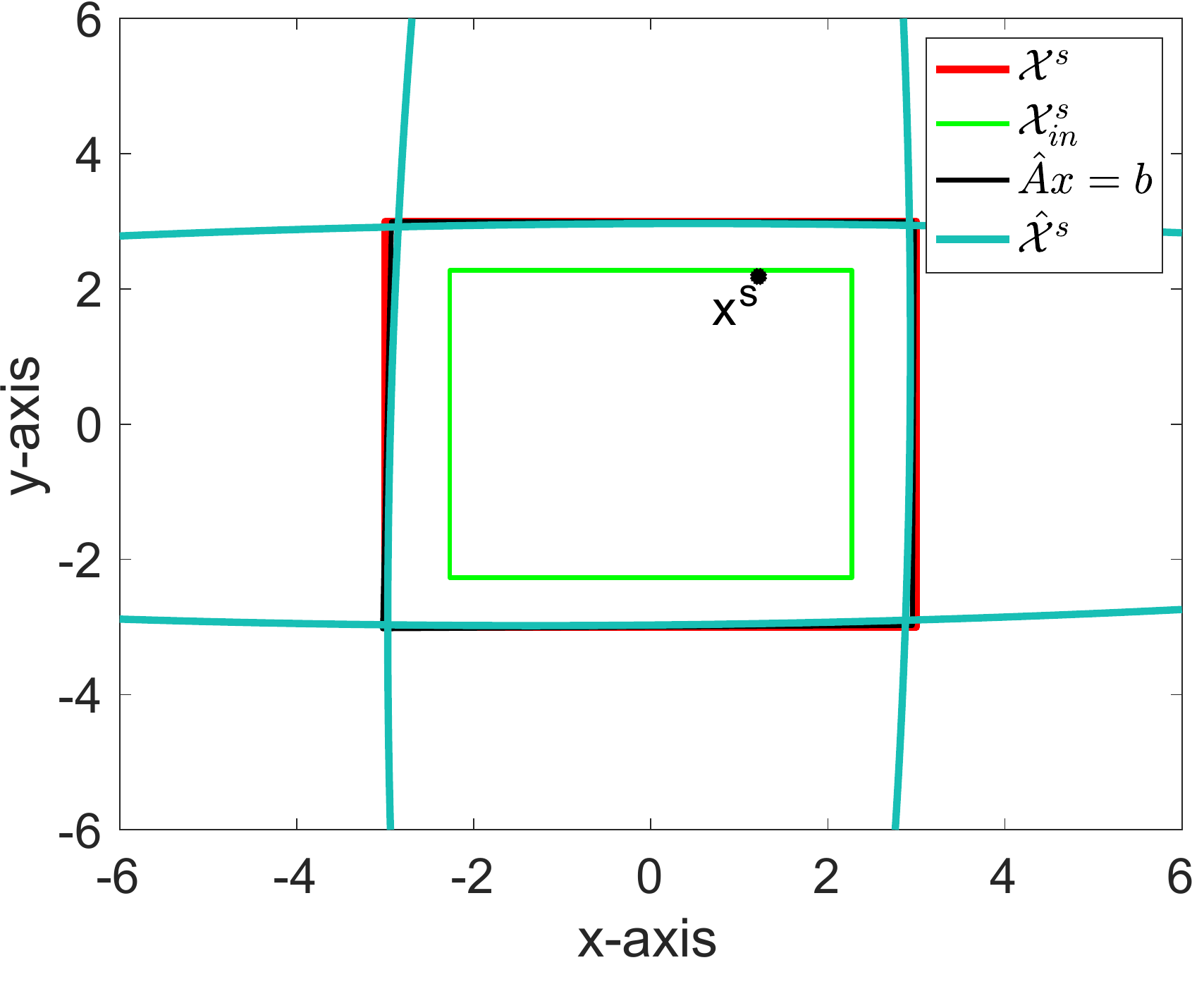}}%
\subcaptionbox{}{\includegraphics[width=0.3\textwidth,trim=120 220 150 230, clip]{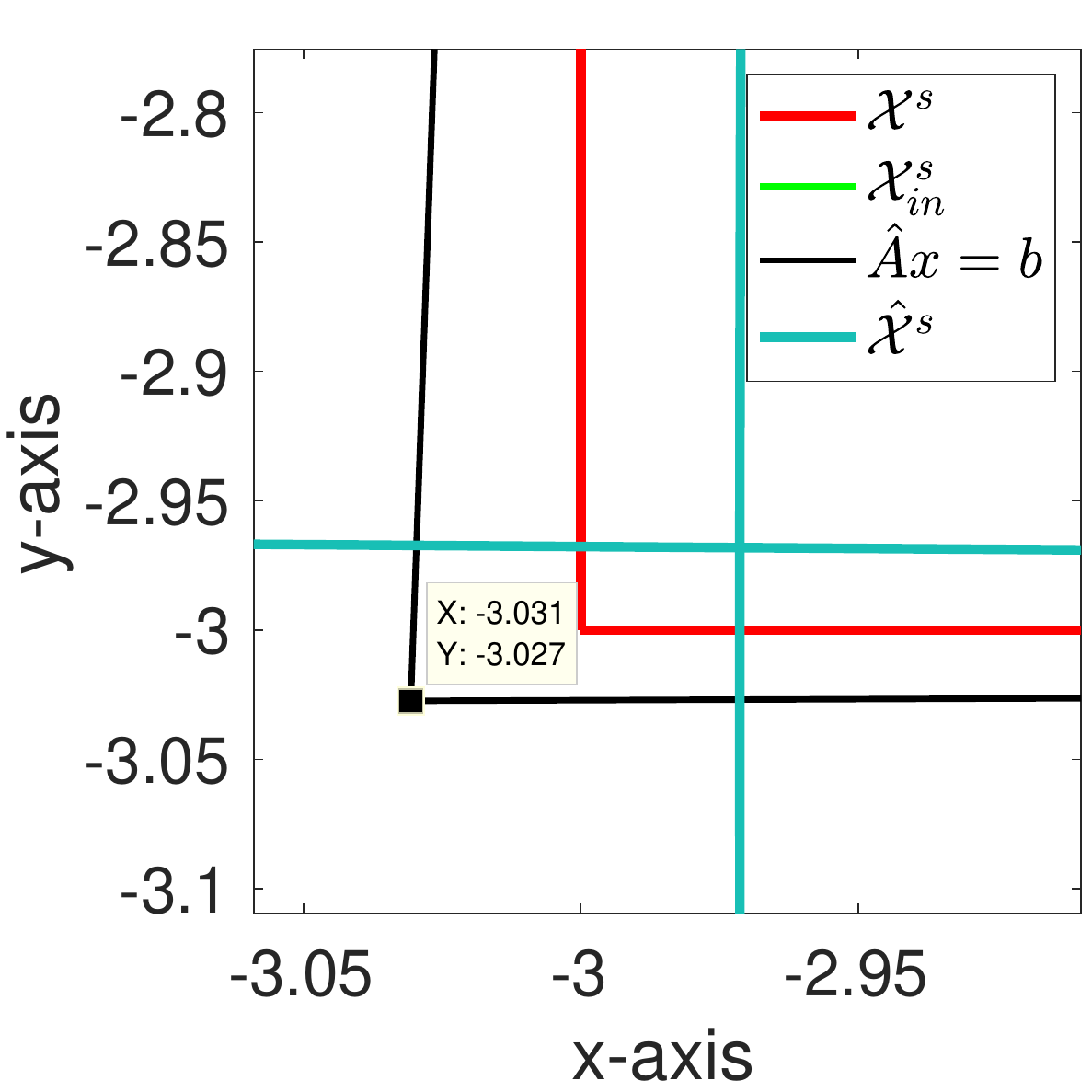}}%
\subcaptionbox{}{\includegraphics[width=0.3\textwidth,trim=120 220 150 230, clip]{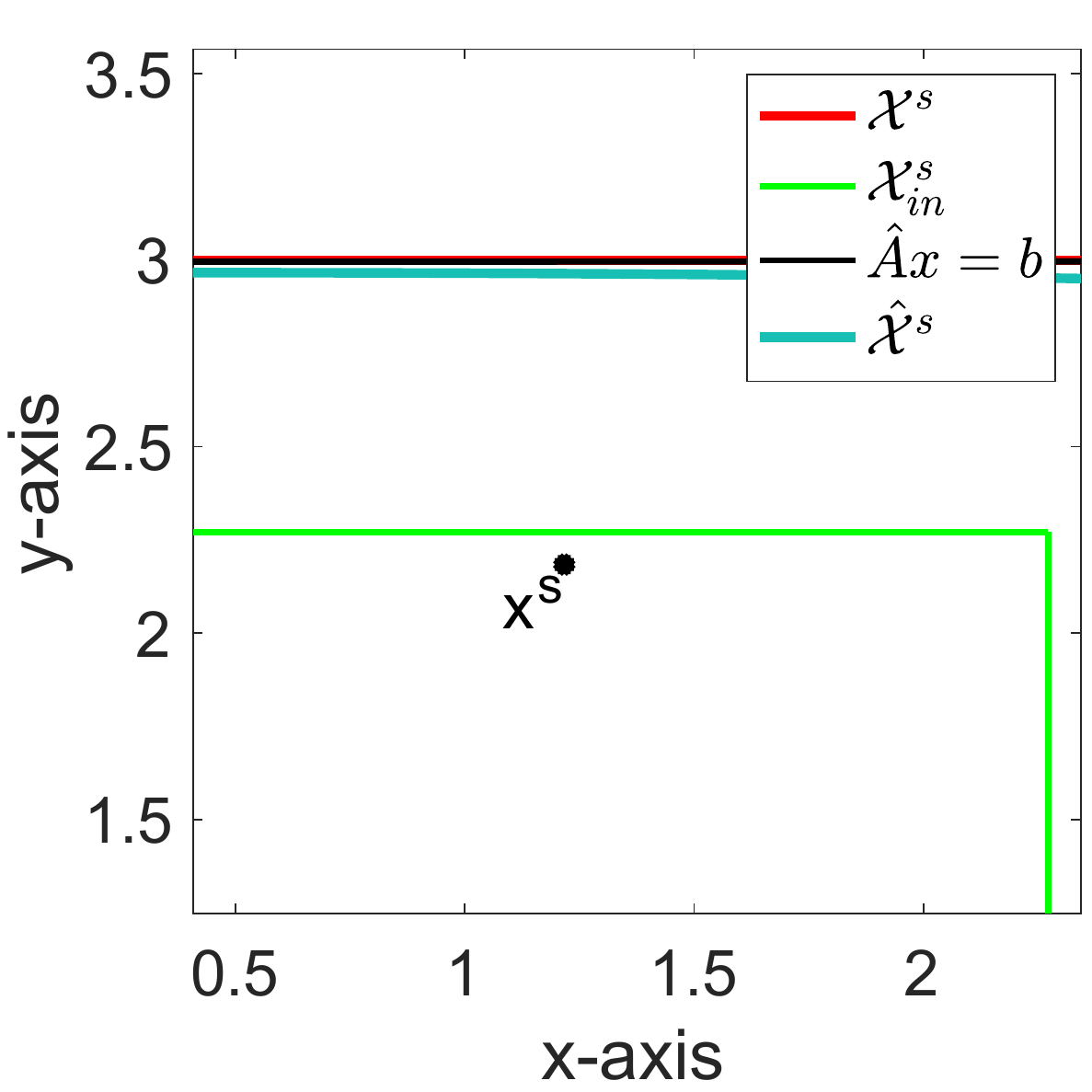}}%
\caption{(a) Shows the true safe set ($\X^s$),  the conservative safe set estimate ($\hat{\mathcal{X}}^{s}$), the `shrunk polytope' ($\X^s_{\text{in}}$), and the naive least squares estimate of the safe set ($\hat{A}x = b$). (b) Shows a zoomed-in view of the third quadrant of (a). (c) Shows a zoomed-in view of the region around $x^s$ to show that it lies inside $\X^s_{\text{in}}$.}
\label{fig:estimation}
\end{figure}

\begin{figure*}[h]
\begin{center}
    \includegraphics[width=0.99\textwidth,,trim=10 140 10 140, clip]{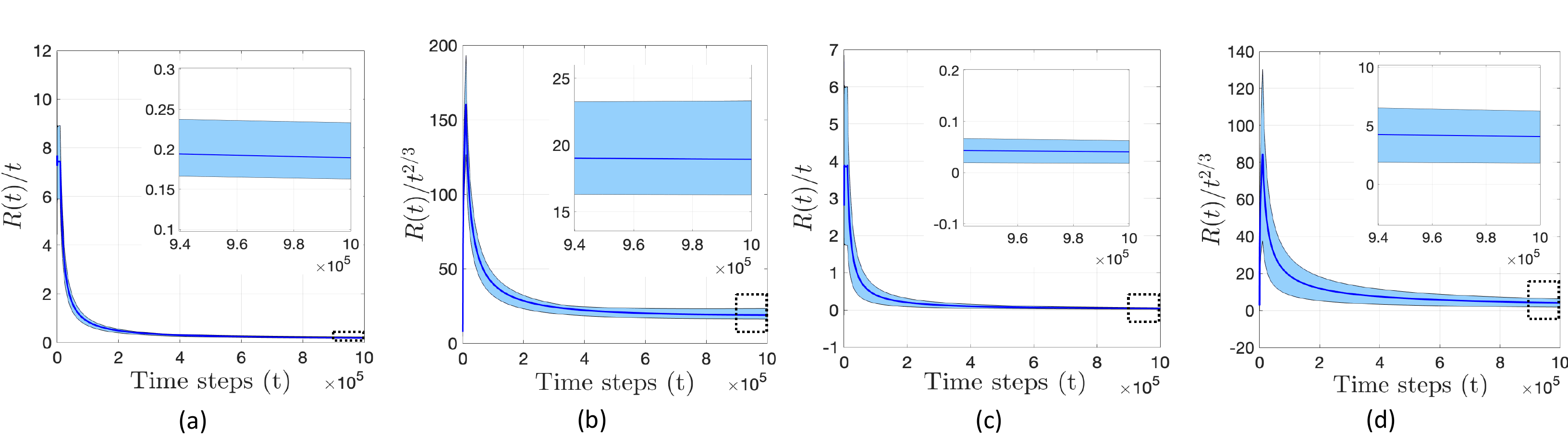}
    \caption{(a) $R(t)/t$ vs $t$ for $f_1$. (b) $R(t)/t^{2/3}$ vs $t$ for $f_1$. (c) $R(t)/t$ vs $t$ for $f_2$. (d)  $R(t)/t^{2/3}$ vs $t$ for $f_2$. All the plots are averaged over 6 random realizations. The light blue region shows the error (max - min) in mean value. The subplots in the top right corner are zoomed-in view into the final $60000$ time steps. $R(t)/t$ decays to zero for both $f_1$ and $f_2$. $R(t)/t^{2/3}$ converges to a constant value $\sim20$ for $f_1$, and $\sim4$ for $f_2$.}
    \label{fig:regret_plots}
\end{center}
\end{figure*}

Fig. \ref{fig:estimation}(a)  shows  the true safe set ($\X^s$),  the conservative safe set estimate ($\hat{\mathcal{X}}^{s}$), and  the `shrunk polytope' ($\X^s_{\text{in}}$). We also show the polytope obtained using the naive least squares estimate, $\{x : \hat{A} x \leq b\}$, where $\hat{A}$ is obtained according to  \eqref{eq:A-hat-estimate}. Please see that $\X^s_{\text{in}} \subset  \hat{\mathcal{X}}^{s}\subset \X^s$, as guaranteed by our results in Lemma \ref{lem:conservative-subset} and Lemma \ref{lem:shrunk-polytope-subset}. It can also be seen that the polytope obtained using the naive least squares estimate need not be a subset of the safe set $\X^s$. We highlight this aspect in Fig. \ref{fig:estimation}(b). So, an OCO algorithm that uses this naive estimate cannot guarantee safety constraint satisfaction at all time steps.  Fig. \ref{fig:estimation}(c) also shows that the safe baseline action $x^{s}$ is inside the `shrunk polytope' $\X^s_{\text{in}}$, as guaranteed by our theory (see the proof of Lemma \ref{lem:shrunk-polytope-subset}). 

\begin{figure}[h!]
    \centering
    \includegraphics[width=0.35\textwidth,trim=10 130 20 140, clip]{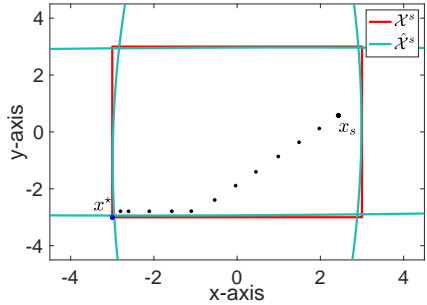}
    \caption{The (sampled) sequence of actions taken by the SO-PGD algorithm in the online projected gradient descent phase. All the actions lie  inside the true safe set $\X^s$.}
    \label{fig:so-pgd-trjectory}
\end{figure}

\paragraph{Online gradient descent:}  Fig. \ref{fig:so-pgd-trjectory} shows the sequence of  actions generated by the SO-PGD algorithm in one experiment. We do not plot all the actions, but only a regularly sampled version of the sequence of actions to avoid crowding the plot. Notice that these actions lie inside the safe set $\mathcal{X}^{s}$

\paragraph{Regret performance:} The regret performance of the SO-PGD algorithm is shown in Fig. \ref{fig:regret_plots}. Instead of plotting regret directly, we plot $R(t)/t$ and $R(t)/t^{2/3}$, where $R(t)$ is the cumulative regret incurred until time $t$. From the figures, it is easy to observe that $R(t)/t$ goes to zero, ensuring that the regret is indeed sublinear. Also, $R(t)/t^{2/3}$ converges to a constant value, indicating that the regret of the SO-PGD algoirthm is indeed $O(T^{2/3})$, as guaranteed by Theorem \ref{thm:main-theorem}.

\section{Conclusion}
In this work,  we addressed the problem  of safe online convex optimization, where the action at each time step must satisfy a set of linear safety constraints. The parameters that specify the linear safety constraints are unknown to the algorithm. We proposed an algorithm called SO-PGD algorithm to solve this problem. Our algorithm comprises of two phases, a safe exploration phase to estimate the unknown safe set and an online gradient descent phase for online optimization. We showed that by carefully balancing the duration of the exploration phase and online optimization phase, the SO-PGD algorithm can achieve $O(T^{2/3})$ regret while satisfying the safety constraints at all times step, with high probability. To the best of our knowledge, this is the first such result in the OCO literature, even in a setting with liner constraints.

In the future, we plan to extend our results to develop projection-free safe OCO algorithms.  We will also investigate if it is possible to achieve $O(T^{1/2})$ regret with no constraint violation, without making any additional strong assumption.


\bibliography{References.bib}
\appendix
\onecolumn
\section{Appendix}

\subsection{Preliminaries}

We use following well known result from linear bandits literature. 
\begin{theorem}[Theorem 2,\cite{abbasi2011improved}]
\label{theorem:ellipsoid_region_yadkori}
Let $\{F_t\}_{t=0}^{\infty}$ be a filtration. Let $\{\eta_t\}_{t=1}^{\infty}$ be a real valued stochastic process such that $\eta_t$ is $F_t$-measurable and $\eta_t$ is conditionally $R$-sub-Guassian for some $R\geq 0$. Let $\{X_t\}_{t=1}^{\infty}$ be an $\R^d$-valued stochastic process such that $X_t$ is $F_{t-1}$-measurable. Let $V_t$ be defined as $\lambda I + \sum^{t}_{t=1} X_{t} X^{\top}_{t}$ for $\lambda>0$. Let, $Y_t = a^{\top}x + \eta_t$. Let $\hat{a}_{t} = V_t^{-1}\sum_{i=1}^{t} Y_{i} X_{i}$ be the $\ell_{2}$-regularized least squares estimate of $a$. Assume that  $\|a\| \leq L_A$ and  $\|X_t\| \leq L, \forall t$. Then, for any $\delta > 0$,  with a probability at least $1-\delta$, the true parameter $a$ lies in the set  
\begin{align}
    C_t &= \Big\{a \in \R^d: \|\hat{a} - a\|_{V_t} \leq R\sqrt{d \log\left(\frac{1+tL^2/\lambda}{\delta} \right)} + \sqrt{\lambda}L_A \Big\},
\end{align}
for all $t\geq1$. 
\end{theorem}

We will use the following results on matrix Chernoff inequality  \cite[Theorem 5.1.1.]{tropp2015introduction}.

\begin{theorem}[Theorem 5.1.1. \cite{tropp2015introduction}] \label{theorem:matrix_chernoff}
Consider a finite sequence $\{X_k\}$ of independent, random, symmetric matrices with a common dimension $d$.  Assume that $
\lambda_{\min}(X_k) \geq 0$ and  $\lambda_{\max}(X_{k}) \leq L, \forall k.$
Introduce the random matrix $Y=\sum_{k}X_k$. Define the minimum eigenvalue $\mu_{\min}$ of the expectation $\mathbb{E}[Y]$ as $\mu_{\min} = \lambda_{\min}(\mathbb{E}[Y]) = \lambda_{\min} (\sum_{k}  \mathbb{E}[X_k])$. then, 
\begin{align}
    \mathbb{P}(\lambda_{\min}(Y) \leq \epsilon\mu_{\min}) \leq d~ \mathrm{e}^{-(1-\epsilon)^2 \mu_{\min}/2L}~~\text{for any}~~\epsilon \in (0,1).
\end{align}
\end{theorem}

\subsection{Proof of Lemma \ref{lem:safe-exploration-lemma-1}}
\begin{proof}
For any $t \in [1, T_{0}]$, and for any $i \in [1, m]$, we have 
\begin{align*}
     a^{\top}_{i} x_{t} &\stackrel{(i)}{=} a^{\top}_{i} ((1 - \gamma) x^{s} + \gamma \zeta_{t})  \\
    &=  (1 - \gamma) a^{\top}_{i} x^{s} + \gamma a^{\top}_{i} \zeta_{t} = (1 - \gamma) b^{s}_{i} + \gamma a^{\top}_{i} \zeta_{t} \stackrel{(ii)}{\leq} (1 - \gamma) b^{s}_{i} + \gamma L_{A} \min\{1, L\}
\end{align*}
 Here, we get $(i)$ by the definition of exploration action  \eqref{eq:safe-exploration-action-1}, and $(ii)$ by using the fact that $\max_{i} \norm{a_{i}} \leq L_{A}$ and $\norm{\zeta_{t}} \leq \min\{1, L\}$. Now, for $x_{t}$ to satisfy the safety constraint, it is sufficient to have $(1 - \gamma) b^{s}_{i} + \gamma  L_{A} \min\{1, L\} \leq b_{i}$, or equivalently,  $\gamma (L_{A} \min\{1, L\} - b^{s}_{i}) \leq (b_{i} - b^{s}_{i})$. This leads to the sufficient condition $\gamma L_{A}  \leq \min_{i}~(b_{i} - b^{s}_{i}) = \Delta^{s}$. 
\end{proof}

\subsection{Proof of Lemma \ref{lem:confidence-ball} and Lemma \ref{lem:conservative-subset}}

\begin{proof}[Proof of Lemma \ref{lem:confidence-ball}]
Using Theorem \ref{theorem:ellipsoid_region_yadkori}, for any $i \in [1, m]$, we get 
\begin{align*}
    \mathbb{P}(a_{i} \in \mathcal{C}_{i}(\delta)) \geq 1 - \delta/m.
\end{align*}
Now, we get the desired result by applying union bound. 
\end{proof}

\begin{proof}[Proof of Lemma \ref{lem:conservative-subset}]

From Lemma \ref{lem:confidence-ball},  $a_{i} \in \mathcal{C}_{i}(\delta)$ for all $i \in [1, m]$ with a probability greater than $(1-\delta)$. Then, by definition,  for any $x \in \hat{\X}^s$, we have $a^{\top}_{i} x \leq b_{i}$ for all $i \in [1, m]$, with probability greater than $1 - \delta$. So,  $\hat{\X}^s \subseteq \X^s$ with probability at least $1-\delta$.
\end{proof}

\subsection{Proof of Proposition \ref{prop:opgd-regret}}

This results follows from the standard regret analysis of online projected gradient descent algorithm \cite[Theorem 3.1]{hazan2016introduction}. We reproduce the result here for completeness.

\begin{proof}
By  convexity of the function $f_t$
\begin{align}
\label{eq:odg-step-1}
    f_{t}(x_{t}) - f_{t}(\hat{x}^{\star}) \leq \nabla f_{t}(x_{t})^{\top} (x_{t} - \hat{x}^{\star}).
\end{align}

We will now upper bound $\nabla f_{t}(x_{t})^{\top} (x_{t} - \hat{x}^{\star})$ as follows:
\begin{align*} 
\norm{x_{t+1} - \hat{x}^{\star}}^{2} &= \norm{ \proj_{\hat{\X}^s} (x_t - \eta
\nabla f_{t}(x_{t})) - \hat{x}^{\star}}^{2} \leq \left\|x_t - \eta \nabla f_{t}(x_{t})-\hat{x}^{\star} \right\|^2 \nonumber \\
&=\|x_t- \hat{x}^{\star}\|^2 + \eta^2
\|\nabla f_{t}(x_{t})\|^2 -2 \eta \nabla f_{t}(x_{t})^\top (x_t - \hat{x}^{\star}),
\end{align*}
where the first inequality is by the Pythagorean theorem. Rearranging and using Assumption \ref{as:cost-function}, we get
\begin{align}
\label{eq:odg-step-2}
    2 \nabla f_{t}(x_{t})^\top (x_t - \hat{x}^{\star})\ \leq&\ \frac{ \|x_t-
\hat{x}^{\star}\|^2-\|x_{t+1} - \hat{x}^{\star}\|^2}{\eta} + \eta G^2.
\end{align}

Using \eqref{eq:odg-step-2} in \eqref{eq:odg-step-1} and taking summation, and using the fact that $\eta = 2L/G \sqrt{T}$ we get
\begin{align} 
 \sum_{t = T_0+1}^{T} \left(f_t(x_t)-f_t(\hat{x}^{\star}) \right ) &\leq  \sum_{t = T_0+1}^{T} \nabla f_{t}(x_{t})^\top (x_t - \hat{x}^{\star}) \leq  \sum_{t = T_0+1}^{T} \frac{ \|x_t-	\hat{x}^{\star}\|^2-\|x_{t+1}-\hat{x}^{\star}\|^2}{2 \eta} + \frac{G^{2}}{2}  \sum_{t = T_0+1}^{T} \eta \nonumber \\
 &\leq   \|x_{1} - \hat{x}^{\star}\|^2 \frac{1}{2 \eta} + \frac{G^{2}}{2}  T \eta \leq 2 L^2 \frac{1}{\eta} + \frac{G^{2}}{2}  T \eta \leq L G \sqrt{T} + L G \sqrt{T} = 2 L G \sqrt{T}.
\end{align}
\end{proof}

\subsection{Proof of Lemma \ref{lem:shrunk-polytope-subset} and Proposition \ref{prop:term-3-regret}}

One key step in proving Lemma \ref{lem:shrunk-polytope-subset} and Proposition \ref{prop:term-3-regret} is to get a high probability lower bound on $\lambda_{\min}(V_{T_{0}}$. We will use the matrix matrix Chernoff inequality for achieving this. We state this result as a lemma below. 

\begin{lemma}
\label{lem:lb-on-lambda-min}
For $T_{0} \geq \frac{8 L^{2}}{\gamma^{2} \sigma^{2}_{\zeta}} \log \frac{d}{\delta}$, we have 
\begin{align}
    \mathbb{P}(\lambda_{\min}(V_{T_{0}}) \geq \lambda +  0.5 \gamma^{2} \sigma^{2}_{\zeta} T_{0}) \geq (1 - \delta).
\end{align}
\end{lemma}

\begin{proof}
For $t \in [1, T_{0}]$, $x_{t} = (1 - \gamma) x^{s} + \gamma \zeta_{t}$, where $\zeta_{t}$s  are  zero mean i.i.d. random vectors such that $\norm{ \zeta_{t}} \leq \min\{1, L\}$ and $\mathbb{E}[\zeta_{t} \zeta^{\top}_{t}] = \sigma^{2}_{\zeta} I$. Let $X_{t} = x_{t} x^{\top}_{t}$. Then, $X_{t}$ is symmetric and positive semidefinite. So, $\lambda_{\min}(X_{t}) \geq 0$. Also, $\lambda_{\max}(X_{t}) \leq \norm{x_{t}}^{2} \leq L^{2}$. We will also get that $\mathbb{E}[X_{t}] = (1 - \gamma)^{2} x^{s}  (x^{s})^{\top} + \gamma^{2} \sigma^{2}_{\zeta} I$.

Let $Y = \sum^{T_{0}}_{t=1} X_{t}$ and $ \mu_{\min} = \lambda_{\min}(\mathbb{E}[Y])$. Then, 
\begin{align}
\label{eq:pflem:lb-on-lambda-min-st1}
    \mu_{\min} = \lambda_{\min}(\mathbb{E}[Y]) = \lambda_{\min}(\sum^{T_{0}}_{t=1} \mathbb{E}[X_{t}]) = \lambda_{\min}(T_{0} ((1 - \gamma)^{2} x^{s}  (x^{s})^{\top} + \gamma^{2} \sigma^{2}_{\zeta} I)) \geq  \gamma^{2} \sigma^{2}_{\zeta} T_{0}.
\end{align}

Now, using the matrix Chernoff inequality stated in Theorem \ref{theorem:matrix_chernoff}, and the above inequality \eqref{eq:pflem:lb-on-lambda-min-st1}, we get
\begin{align}
    \mathbb{P}(\lambda_{\min}(Y) \leq \epsilon \gamma^{2} \sigma^{2}_{\zeta} T_{0}) \leq \mathbb{P}(\lambda_{\min}(Y) \leq \epsilon \mu_{\min}) \leq d ~ \text{exp}\left(- \frac{(1-\epsilon)^2 \mu_{\min}}{ 2 L^{2}} \right) \leq d ~ \text{exp}\left(- \frac{(1-\epsilon)^2 \gamma^{2} \sigma^{2}_{\zeta} T_{0}}{ 2 L^{2}} \right)
\end{align}

For $\epsilon = 1/2$, with $T_{0} \geq \frac{8 L^{2}}{\gamma^{2} \sigma^{2}_{\zeta}} \log \frac{d}{\delta}$, we get 
\begin{align}
    \mathbb{P}(\lambda_{\min}(Y) \geq 0.5 \gamma^{2} \sigma^{2}_{\zeta} T_{0}) \geq (1 - \delta).
\end{align}

Since $V_{T_{0}} = \lambda I + \sum^{T_{0}}_{t=1} x_{t} x^{\top}_{t} = \lambda I + Y$, we have  $\lambda_{\min}(V_{T_{0}}) \geq \lambda + \lambda_{\min}(Y)$. This will give,
\begin{align}
    \mathbb{P}(\lambda_{\min}(V_{T_{0}}) \geq \lambda +  0.5 \gamma^{2} \sigma^{2}_{\zeta} T_{0}) \geq (1 - \delta).
\end{align}

\end{proof}

Consider the events
\begin{align}
    \mathcal{E}_{A} = \{a_{i} \in \mathcal{C}_{i}(\delta), \forall i \in [1, m]\},~~ \mathcal{E}_{\lambda} = \{\lambda_{\min}(V_{T_{0}}) \geq \lambda +  0.5 \gamma^{2} \sigma^{2}_{\zeta} T_{0}\},~~ \mathcal{E} = \mathcal{E}_{A} \cap \mathcal{E}_{\lambda}. 
\end{align}
From Lemma \ref{lem:confidence-ball}, $\mathbb{P}( \mathcal{E}_{A}) \geq (1 - \delta)$. From Lemma \ref{lem:lb-on-lambda-min}, with  $T_{0} \geq \frac{8 L^{2}}{\gamma^{2} \sigma^{2}_{\zeta}} \log \frac{d}{\delta}$, $\mathbb{P}( \mathcal{E}_{\lambda}) \geq (1 - \delta)$. Then, using union bound, $\mathbb{P}(\mathcal{E}) \geq 1 - 2 \delta$. Our analysis for the proof of  Lemma \ref{lem:shrunk-polytope-subset} and Proposition \ref{prop:term-3-regret} will be conditioned on the event $\mathcal{E}$. So, they will be true with a probability greater than $(1 - 2 \delta)$. \\

We now give the proof of Lemma \ref{lem:shrunk-polytope-subset}.

\begin{proof}[Proof of Lemma \ref{lem:shrunk-polytope-subset}]

To show that $\X^{s}_{\textnormal{in}}$ is non-empty, we will show that $x^{s}$ is an element of $\X^{s}_{\textnormal{in}}$ for $T_{0} \geq  \frac{8 \beta^{2}_{T}(\delta) L^{2}}{\gamma^{2}\sigma^{2}_{\zeta} (\Delta^{s})^{2}}$. For $x^{s}$ to be an element of $\X^{s}_{\textnormal{in}}$, we need 
\begin{align*}
    a^{\top}_{i} x_{s} + \tau_{\textnormal{in}} \leq b_{i}, \forall i \in [1, m]~~ \implies ~~  \tau_{\textnormal{in}} \leq \min_{i} ~(b_{i} -   a^{\top}_{i} x_{s}) = \min_{i}~(b_{i} - b^{s}_{i}) = \Delta^{s}. 
\end{align*}

For $\tau_{\textnormal{in}} = {2\beta_{T_0}(\delta)L}/{\sqrt{\lambda_{\min}({V_{T_0})}}}$, this is equivalent to satisfying the condition $  \lambda_{\min}({V_{T_0}}) \geq \frac{4 \beta^{2}_{T_0}(\delta) L^{2}}{(\Delta^{s})^{2}}$. Now, conditioned on the event $\mathcal{E}$, this  inequality  is satisfied with  a probability greater than $(1- 2 \delta)$ if $\lambda +  0.5 \gamma^{2} \sigma^{2}_{\zeta} T_{0} \geq \frac{4 \beta^{2}_{T_0}(\delta) L^{2}}{(\Delta^{s})^{2}}$, which is guaranteed for any $T_{0}$ such that
\begin{align}
\label{eq:T-0-LB}
    T_{0} \geq \frac{8 \beta^{2}_{T}(\delta) L^{2}}{\gamma^{2}\sigma^{2}_{\zeta} (\Delta^{s})^{2}}.
\end{align}
Please note that the above lower bound on $T_{0}$ also satisfies the lower bound condition for the result of Lemma \ref{lem:lb-on-lambda-min} to be true when $\Delta^{s}$ is small or $T$ is large, which is typically the case. So, when  $T_{0}$ satisfies the condition \eqref{eq:T-0-LB},  $x^{s} \in \X^{s}_{\textnormal{in}}$, and hence  $\X^{s}_{\textnormal{in}}$ is non-empty.

To show that $\X^{s}_{\textnormal{in}} \subset \hat{\mathcal{X}}^{s}$, consider an arbitrary $x \in \X^{s}_{\textnormal{in}}$. Then, by definition, $a^{\top}_{i} x+  \frac{2\beta_{T_0}(\delta)L}{\sqrt{\lambda_{\min}({V_{T_0})}}} \leq b_{i}$. Now, 
\begin{align}
    \hat{a}^{\top}_{i} x + \beta_{T_0}(\delta) \left\|x \right\|_{V_{T_0}^{-1}} &= a^{\top}_{i} x + (\hat{a}^{\top}_{i} x - a^{\top}_{i} x)  + \beta_{T_0}(\delta) \left\|x \right\|_{V_{T_0}^{-1}} \\ &\leq a^{\top}_{i} x +  \norm{ \hat{a}_{i} - a_{i}  }_{V_{T_0}} \norm{x}_{V^{-1}_{T_0}} + \beta_{T_0}(\delta) \left\|x \right\|_{V_{T_0}^{-1}} \nonumber \\
   &\stackrel{(i)}{\leq} a^{\top}_{i} x + 2 \beta_{T_0}(\delta) \left\|x \right\|_{V_{T_0}^{-1}} \nonumber \\
    &\leq  a^{\top}_{i} x +  \frac{2 \beta_{T_0}(\delta) \norm{x}_{2} }{\sqrt{\lambda_{\min}(V_{T_0})}} \stackrel{(ii)}{\leq} a^{\top}_{i} x +  \frac{2 \beta_{T_0}(\delta) L}{\sqrt{\lambda_{\min}(V_{T_0})}} \leq b_{i},
\end{align}
where we get $(i)$ conditioned on the event $\mathcal{E}$ and $(ii)$ by using the fact that $\norm{x}_{2} \leq L$. This implies that $x \in \hat{\mathcal{X}}^{s}$. Since $x \in \X^{s}_{\textnormal{in}}$ is arbitrary, we get $\X^{s}_{\textnormal{in}} \subset \hat{\mathcal{X}}^{s}$. This also immediately implies that $\norm{\Pi_{\hat{\mathcal{X}}^{s}}(x^{\star}) - x^{\star}} \leq  \norm{\Pi_{\X^{s}_{\textnormal{in}}}(x^{\star}) - x^{\star}}$.
\end{proof}

We now give the proof of Proposition \ref{prop:term-3-regret}

\begin{proof}[Proof of Proposition \ref{prop:term-3-regret}]
Conditioned on the event $\mathcal{E}$,
\begin{align*}
\sum^{T}_{t=T_{0}+1} f_t(\hat{x}^\star) - f_t(x^\star) &\stackrel{(i)}{\leq}  G (T - T_{0}) \norm{\hat{x}^\star - x^{\star}} =  G (T - T_{0}) \norm{\Pi_{\hat{\mathcal{X}}^{s}}(x^{\star}) - x^{\star}}  \\
&\stackrel{(ii)}{\leq} G T \norm{\Pi_{\X^{s}_{\textnormal{in}}}(x^{\star}) - x^{\star}}  \stackrel{(iii)}{\leq} G T \frac{\sqrt{d}}{C(A, b)}  \tau_{\textnormal{in}}  \stackrel{(iv)}{=} G T \frac{\sqrt{d}}{C(A, b)}   \frac{2\beta_{T_0}(\delta)L}{\sqrt{\lambda_{\min}({V_{T_0})}}},
\end{align*}
where $(i)$ is by using Assumption \ref{as:cost-function}, $(ii)$ from Lemma \ref{lem:shrunk-polytope-subset}, $(iii)$ is by using  Lemma \ref{lem:fereydounian2020safe-lemma}, and $(iv)$ is by applying the value of $\tau_{\textnormal{in}}$ used in Lemma \ref{lem:shrunk-polytope-subset}.

Also, conditioned on the event $\mathcal{E}$,  we have $\lambda_{\min}(V_{T_{0}}) \geq \lambda +  0.5 \gamma^{2} \sigma^{2}_{\zeta} T_{0}$. Using this in the above inequality, we get
\begin{align*}
    \sum^{T}_{t=T_{0}+1} f_t(\hat{x}^\star) - f_t(x^\star) \leq G T \frac{\sqrt{d}}{C(A, b)}   \frac{2\beta_{T_0}(\delta)L}{\sqrt{\lambda +  0.5 \gamma^{2} \sigma^{2}_{\zeta} T_{0}}} \leq G T \frac{\sqrt{d}}{C(A,b)}   \frac{2\beta_{T_0}(\delta)L}{\sqrt{0.5 \gamma^{2} \sigma^{2}_{\zeta} T_{0}}}
\end{align*}
Reordering the terms, we get the stated result.

\end{proof}

\subsection{Proof of Theorem \ref{thm:main-theorem}}

\begin{proof}[Proof of Theorem \ref{thm:main-theorem}]
We first prove the safety guarantee. For $t \in [1, T_{0}]$, $x_{t} \in \mathcal{X}^{s}$  by Lemma \ref{lem:safe-exploration-lemma-1}. For $t > T_{0}$, the SO-PGD algorithm performs online projected gradient descent with respect to the set $\hat{\X}^{s}$.  So $x_{t} \in \hat{\X}^{s}$ for $t \in [T_{0}+1, T]$. Now, by Lemma \ref{lem:conservative-subset}, $\hat{\X}^{s} \subset \mathcal{X}^{s}$ with a probability greater than $(1-\delta)$. So, $x_{t} \in \mathcal{X}^{s}, \forall t \in [1, T],$ with a probability greater than $(1 - \delta)$. \\

We now prove the regret bound. From the regret decomposition in  \eqref{eq:SO-PGD_regret_decompos}, we have $R(T) = \textnormal{Term I} + \textnormal{Term II} + \textnormal{Term III}$. Using the upper bound for Term I from \eqref{eq:regret-tem-1}, the upper bound for Term II from Proposition \ref{prop:opgd-regret}, and the upper bound for Term III from Proposition \ref{prop:term-3-regret}, we get
\begin{align}
\label{eq:mainthm-pf-st1}
    R(T) &\leq 2 L G T_{0} + 2 L G \sqrt{T} +\frac{L G \sqrt{8 d}}{C(A,b) \sqrt{\gamma^{2} \sigma^{2}_{\zeta}}}   \frac{\beta_{T}(\delta) T}{\sqrt{  T_{0}}},
\end{align}
with a probability greater than $(1 - 2 \delta)$, for $T_{0} \geq \frac{8 \beta^{2}_{T}(\delta) L^{2}}{\gamma^{2}\sigma^{2}_{\zeta} (\Delta^{s})^{2}}$.  We will now select $T_{0} = T^{2/3}$. To ensure the lower bound condition on $T_{0}$ given in Proposition  \ref{prop:term-3-regret}, it is sufficient to have 
$T_{0} = T^{2/3} \geq \frac{8 \beta^{2}_{T} L^{2}}{\gamma^{2}\sigma^{2}_{\zeta} (\Delta^{s})^{2}}$, which is equivalent to having 
\begin{align}
\label{eq:mainthm-pf-st2}
    T \geq \left( \frac{\sqrt{8} \beta_{T}(\delta) L}{\gamma \sigma \Delta^{s}} \right)^{3}.
\end{align}

Now, using $T_{0} = T^{2/3}$ in \eqref{eq:mainthm-pf-st1}, we get
\begin{align}
\label{eq:mainthm-pf-st3}
    R(T) &\leq 2 L G T^{2/3} + 2 L G \sqrt{T}   +\frac{L G \sqrt{8 d}}{C(A,b) \sqrt{\gamma^{2} \sigma^{2}_{\zeta}}} \beta_{T}(\delta) T^{2/3}.
\end{align}

\end{proof}

\subsection{Additional Simulation Results}

\paragraph{Resource Allocation with Safety Constraints:} Here, we consider the cost function that together with a set of linear inequality constraints of the form $Ax \leq b$ is a representative of problems arising in resource allocation \citep{yu2020low, ibaraki1988resource}. In particular, we consider the cost functions $\{f_{t}\}$ used in \citep{yu2020low}, where $f_{3,t} = c_t^\top x$ and  $c_t = c_{1,t} + c_{2,t} + c_{3,t}$. Here, each component of  $c_{1,t}$ is uniformly sampled from the interval $\left[-t^{1 / 10},+t^{1 / 10}\right]$; each component of $c_{2,t}$ is uniformly sampled from $\left[-1,0\right]$; and  $c_{3,t}(i) = \left(-1\right)^{p(t)}, \forall i \in [1, d]$, where $p(t)$ is a random permutations of the integers  in the set $[1,T]$.   We use $d=2$, and the same constraint polytope used in Section \ref{sec:simulation-main} formed by $A = [1,0; -1,0; 0,1; 0,-1]$ and $b = x_{\max}\times[1;1;1;1]$, where we choose $x_{\max} = 3$. We now include regret plots for $T=10^5$ time steps for $f_3$ in Fig. \ref{fig:neely_regret_plots}.  

We perform an additional experiment to show that our algorithm works well for different  safe action sets. We choose the same cost functions of the form $f_3$ described above and a triangular shaped true safe set $\X^s = \{x \in \R^2: Ax \leq b\}$ such that  $A = [1,~1;~-1,~0;~0,~-1]$ and $b = [1,~0,~0]^\top$. We choose $x^s = [0.25,0.25]^\top$. We run SO-PGD for this setup for $T=10^4$ time steps. The results from this experiment are recorded in Fig. \ref{fig:neely_triangle}. We observe, as before, that all the exploratory actions (represented by blue circular region around $x^s$) are safe (see Fig. \ref{fig:neely_triangle}.(a)). The whole optimization trajectory lies inside $\X^s$ (see Fig. \ref{fig:neely_triangle}.(b)). The regret performance is shown in Fig. \ref{fig:neely_triangle}.(c), and Fig. \ref{fig:neely_triangle}.(d). As expected, $R(t)/t \rightarrow 0$ and $R(t)/t^{2/3}$ tends to a constant value.

 

\begin{figure*}[h]
\begin{center}
    \includegraphics[width=1\textwidth,trim=10 130 10 170, clip]{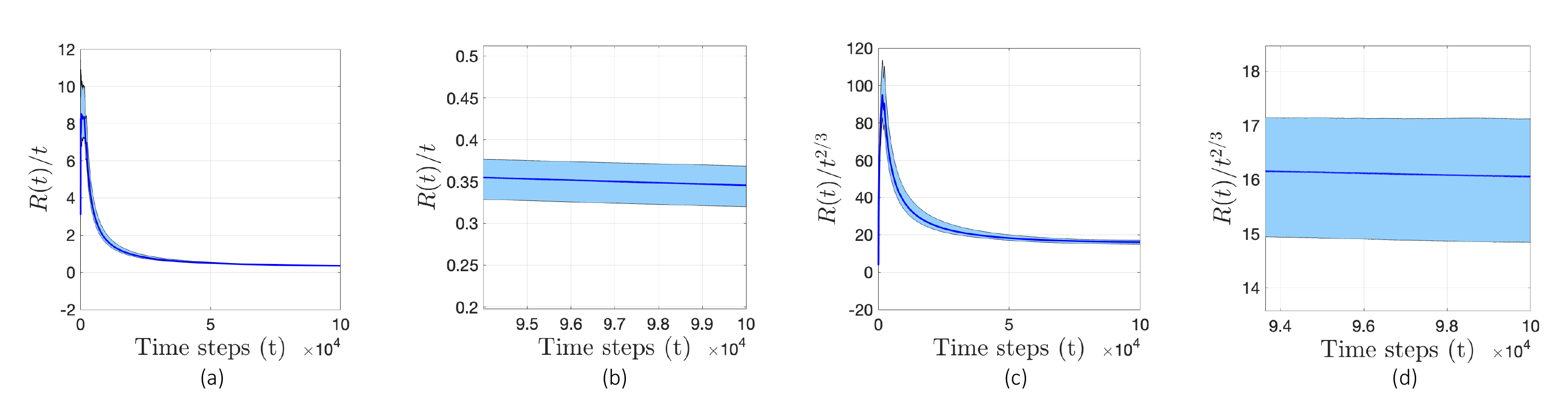}
    \caption{(a) $R(t)/t$ vs $t$ for $f_3$. (b) Zoomed in version of $R(t)/t$ vs $t$ for $f_3$. (c) $R(t)/t^{2/3}$ for $f_3$. (d)  Zoomed in version of $R(t)/t^{2/3}$ vs $t$ for $f_3$. All the plots are averaged over 4 random realizations. The light blue region shows the error (max - min) in mean value.}
    \label{fig:neely_regret_plots}
\end{center}
\end{figure*}
\begin{figure*}[h!]
\begin{center}
    \includegraphics[width=1\textwidth,trim=10 140 10 165, clip]{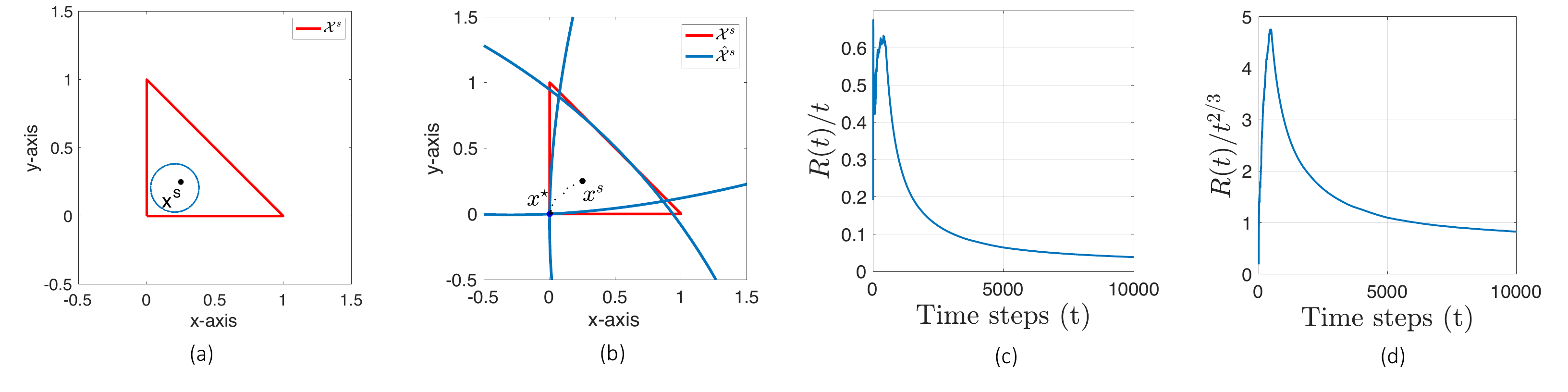}
    \caption{(a) Shows $\X^s$ and exploratory actions.  (b) Shows estimated safe set, it is the interior of blue curves. The black dots from $x^s$ to $x^\star$ denote the optimization trajectory. (c) $R(t)/t$ for $f_3$ for triangular safe set. (d) $R(t)/t^{2/3}$ for $f_3$ for triangular safe set.}
    \label{fig:neely_triangle}
\end{center}
\end{figure*}

\paragraph{Data center scheduling using electricity market price:} Here, we consider a problem motivated from a real-world application  using real-world data. We note that this is the safe OCO version of the  problem considered in \citep{wei2020online}.

The goal is to minimize the cost and maximize the service in a geographically distributed data center.  Let the data center consist of five servers, each of which are located in the following zones of New York: Genesee, Central, North, Mohawk Valley and West, denoted by the index $k=[1, 5]$, respectively. Each zone has different hourly rates for electricity, known as Location Based Marginal Price (LBMP). We obtain the  LBMP data (\verb+$/MWhr+ rates) from the New York Independent System Operator (ISO)'s publicly available energy market and operational data webpage ({\color{blue}\href{https://www.nyiso.com/energy-market-operational-data}{https://www.nyiso.com/energy-market-operational-data}}).  At any hour $t \leq T$, the objective of the data center scheduling program is to choose an action $x_t \in [0,30]^5$ that determines the following:
\begin{enumerate}
    \item The cumulative cost of electricity used in an hour by all the zones, given by $h_{t,1} = c_t^\top x_t$ where $c_t \in \R^5$ is obtained from LBMP data of the five zones.
    \item The cumulative number of jobs served by all the zones, given by $h_{t,2} = \sum_{k=1}^5 \mathbb{E}[D_k(x_t^k)]$. Here, $D_k(\cdot)$ is a Pareto (power law) distribution of mean $8 \log(1+4x_t^k)$ where $x_t^k$ is the $k$th dimension of the action vector $x_{t}$ (i.e., the action corresponding to $k^\text{th}$ zone). 
\end{enumerate}
Let the number of incoming jobs per hour be given by a Poisson distribution with mean equal to 100.
Then, the net objective is to obtain a sequence of $x_t$ such that
\begin{align}
   \sum_{t=1}^T f_{t}(x_{t}) =  \sum_{t=1}^T \Big( c^\top x_t + \lambda (100 - \sum_{k=1}^5 8 \log(1+4x_t^k)) \Big) \label{eq:our_electricity_min}
\end{align}
Here, $\lambda$ is a known constant that balances the objectives of electricity cost minimization and job service maximization. For $T=10000$, we choose $\lambda = 5.7720$. Note that different from \citep{wei2020online}, we include $h_{t,2}$ in the objective rather than as a separate constraint. The safe set of actions is unknown to the scheduler and is required to be learned in a safe manner

We run our  SO-PGD algorithm for this data center scheduling problem with $f_t(x_t) = c^\top x_t + \lambda (100 - \sum_{k=1}^5 8 \log(1+4x_t^k))$, and include the results in Fig. \ref{fig:electricity_pricing_data}. SO-PGD does not violate any constraints while balancing the average money spent and the average number of unserved jobs. 
\begin{figure*}[h!]
\begin{center}
    \includegraphics[width=0.99\textwidth,trim=10 180 0 170, clip]{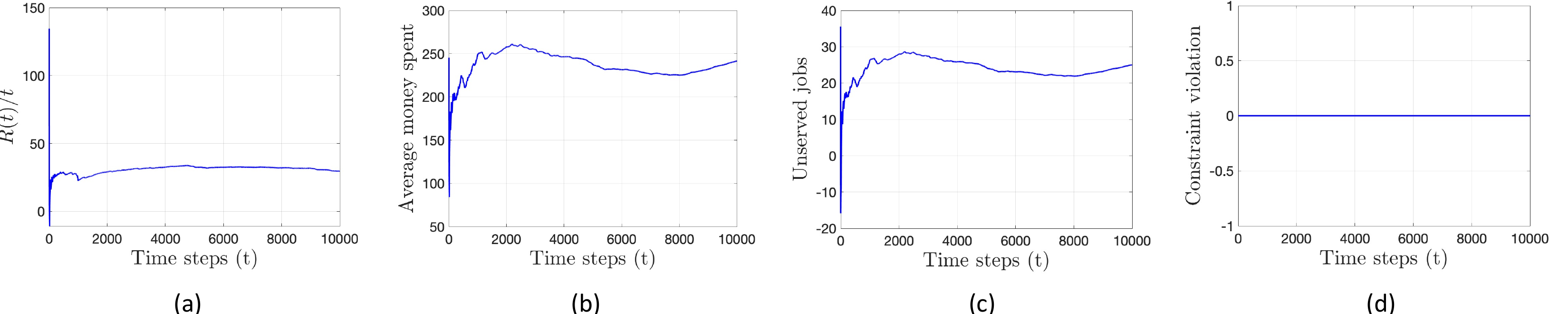}
    \caption{(a) Cumulative regret of $f_t$ over T=10000 time steps. (b) Average money spent as a function of time. (c) Number of average unserved jobs as a function of time. (d) Number of constraint violations as a function of time. Note that, at any time step $t$, our algorithm achieves zero violation of constraint.}
    \label{fig:electricity_pricing_data}
\end{center}
\end{figure*}

\end{document}